\documentclass{ifacconf}

\usepackage{graphicx}
\usepackage{natbib}
\usepackage{amsmath,amssymb,mathtools}
\usepackage{bm}
\usepackage{subfig}
\usepackage{booktabs}
\usepackage{multicol}
\usepackage{multirow}

\DeclarePairedDelimiter{\norm}{\lVert}{\rVert}

\newtheorem{property}{Property}
\newtheorem{theorem}{Theorem}
\newtheorem{assumption}{Assumption}

\newtheorem{lemma}{Lemma}

\newtheorem{proof}{Proof}

\begin{document}
\begin{frontmatter}

\title{Neural Network-Based Adaptive Event-Triggered Control for Dual-Arm Unmanned Aerial Manipulator Systems \thanksref{footnoteinfo}} 

\thanks[footnoteinfo]{This work was supported in part by National Natural Science Foundation of China under Grant 62273187, Grant 62233011, and Grant U25A200563, and in part by the Key Technologies R \& D Program of Tianjin under Grant 23YFZCSN00060. \emph{(Corresponding author: Xiao Liang)} }



\author[1,2]{Yang Wang}
\author[1,2]{Hai Yu}
\author[1,2]{Wei He}
\author[1,2]{Jianda Han}
\author[1,2]{Yongchun Fang}
\author[1,2]{Xiao Liang}

\address[1]{Institute of Robotics and Automatic Information System, College of Artificial Intelligence, and Tianjin Key Laboratory of Intelligent     Robotics, Nankai University, Tianjin 300350, China.}
\address[2]{Institute of Intelligence Technology and Robotic Systems, Shenzhen Research Institute of Nankai University, Shenzhen 518083, China. (e-mail: \{wangy1893, yuhai, howei\}@mail.nankai.edu.cn, \{hanjianda, fangyc, liangx\}@nankai.edu.cn)}

\begin{abstract}                
This paper investigates the control problem of dual-arm unmanned aerial manipulator systems (DAUAMs). Strong coupling between the dual-arm and the multirotor platform, together with unmodeled dynamics and external disturbances, poses significant challenges to stable and accurate operation. An adaptive event-triggered control scheme with neural network-based approximation is proposed to address these issues while explicitly considering communication constraints. First, a dynamic model of the DAUAM system is derived, and a command-filter-based backstepping framework with error compensation is constructed. Then, a neural network is employed to approximate external frictions, and an event-triggered mechanism is designed to reduce the transmission frequency of control updates, thereby alleviating communication and energy burdens. Lyapunov-based analysis shows that all closed-loop signals remain bounded and that the tracking error converges to a neighborhood of the desired trajectory within a fixed time. Finally, experiments on a self-built DAUAM platform demonstrate that the proposed approach achieves accurate trajectory tracking.
\end{abstract}


\end{frontmatter}

\section{Introduction}
Unmanned aerial vehicles (UAVs) have revolutionized operations in three-dimensional space, offering unparalleled maneuverability, hovering capability, and flexibility. These attributes have enabled their widespread adoption across diverse applications, including industrial inspection, security surveillance, and environmental monitoring (\cite{uav1, uav2}). Building upon these advantages, the integration of robotic manipulators has transformed UAVs into unmanned aerial manipulators (UAMs). This evolution empowers them to perform active physical interactions with the environment, facilitating tasks such as scientific sampling, transportation of hazardous materials, aerial maintenance, and contact-based inspections like wall sensing or high-altitude cleaning (\cite{uam1, uam2, uam3}).

Driven by advancements in unmanned systems, research has progressively shifted towards dual-arm UAM configurations. These systems synergize the extensive operational workspace of dual-arm robots with the vertical take-off and landing capabilities of UAVs, demonstrating superior performance over single-arm counterparts in complex, interactive domains such as search and rescue, object manipulation, and aerial assembly. Recent contributions highlight this trend: \cite{dual1} develope a lightweight, low-inertia dual-arm manipulator with a center-of-gravity balancing mechanism, utilizing prismatic joints to extend the workspace and minimize stability impact on the UAV. \cite{dual2} introduce a novel image-based visual-impedance control scheme, enabling effective physical interaction for a DAUAM equipped with vision and force/torque sensors. Focusing on system design, \cite{dual3} create a flexible DAUAM with a low-weight, low-inertia humanoid arm structure. Furthermore, \cite{dual5} establish a bilateral teleoperation system using a leader-follower configuration, allowing an operator to control a DAUAM intuitively, with validation conducted in ground, in-flight, and cable-suspended scenarios. In a specialized application, \cite{dual4} presente an aerial-deployable dual-arm system for power line maintenance, where the arms are deployed as a rolling robot to improve task accuracy and reduce line load. 

The primary challenge associated with DAUAMs lies in the inherent coupling between the arm motions and the multirotor platform. Fast arm movements inevitably disturb the vehicle's attitude and position, and these effects cannot be directly compensated because they depend on joint velocities and accelerations that are typically unmeasured. Moreover, unmodeled dynamics and external disturbances in real-world environments introduce additional uncertainties and risks to the system. These factors make it nontrivial to design controllers that ensure accurate trajectory tracking and safe operation for transportation tasks.

To cope with such nonlinear and uncertain dynamics, the backstepping technique has been widely employed as an effective controller-design framework, and has achieved considerable success in system analysis and control synthesis (\cite{bs1,bs2}). However, the repetitive differentiation of virtual control inputs in the backstepping procedure gives rise to the well-known explosion of complexity problem. To alleviate this issue, dynamic surface control was first proposed in \cite{DynamicSurface} by incorporating a first-order filter into the backstepping design. Nonetheless, the presence of filter-induced errors may still degrade system performance. In \cite{CommandFilter}, a command-filter-based backstepping scheme with an explicit error-compensation mechanism is developed to mitigate these adverse effects. On the other hand, from the perspective of communication efficiency, continuously sampling and transmitting control signals to onboard actuators in real time can lead to substantial communication and energy costs. Thus, how to reduce the communication burden while preserving satisfactory control performance becomes a critical problem that must be carefully addressed.

To address the abovementioned challenges, this paper proposes an adaptive event-triggered control scheme for DAUAMs with neural network (NN) approximation. The main contributions of this paper are summarized as follows:
\begin{enumerate}
\item A novel switching-function-based command-filtered backstepping framework is constructed to handle the challenges brought by the nonlinear and strongly coupled dynamics of the DAUAMs. 
By appropriately designing the switching function, the error-compensation term is rendered continuous, 
so that the virtual control laws remain differentiable and the validity of the command filter is guaranteed.

\item To cope with the time-varying external frictions and limited communication resources in DAUAMs, a neural network-based event-triggered mechanism is further integrated into the controller. 
Notably, the neural network adaptively estimates the unknown external friction forces, while the event-triggered update rule reduces the communication burden. 
Rigorous Lyapunov analysis shows that all closed-loop signals are uniformly bounded and that the tracking error converges to a small neighborhood of the desired trajectory within a fixed finite time.
\end{enumerate}

The remainder of this paper is organized as follows. Section~\ref{sec:mod} introduces the system model and provides the theoretical basis for the subsequent controller design. Section~\ref{sec:controller} presents the detailed procedure for controller synthesis, while Section~\ref{sec:stability} carries out the corresponding stability analysis. In Section~\ref{sec:exp}, experiments on a self-built DAUAM platform are conducted to validate the proposed method. Finally, Section~\ref{sec:con} concludes the paper and outlines directions for future work.

\section{System Modeling and Normalization}\label{sec:mod}

\begin{figure}[htbp]
\centering
\includegraphics[width=3.3in]{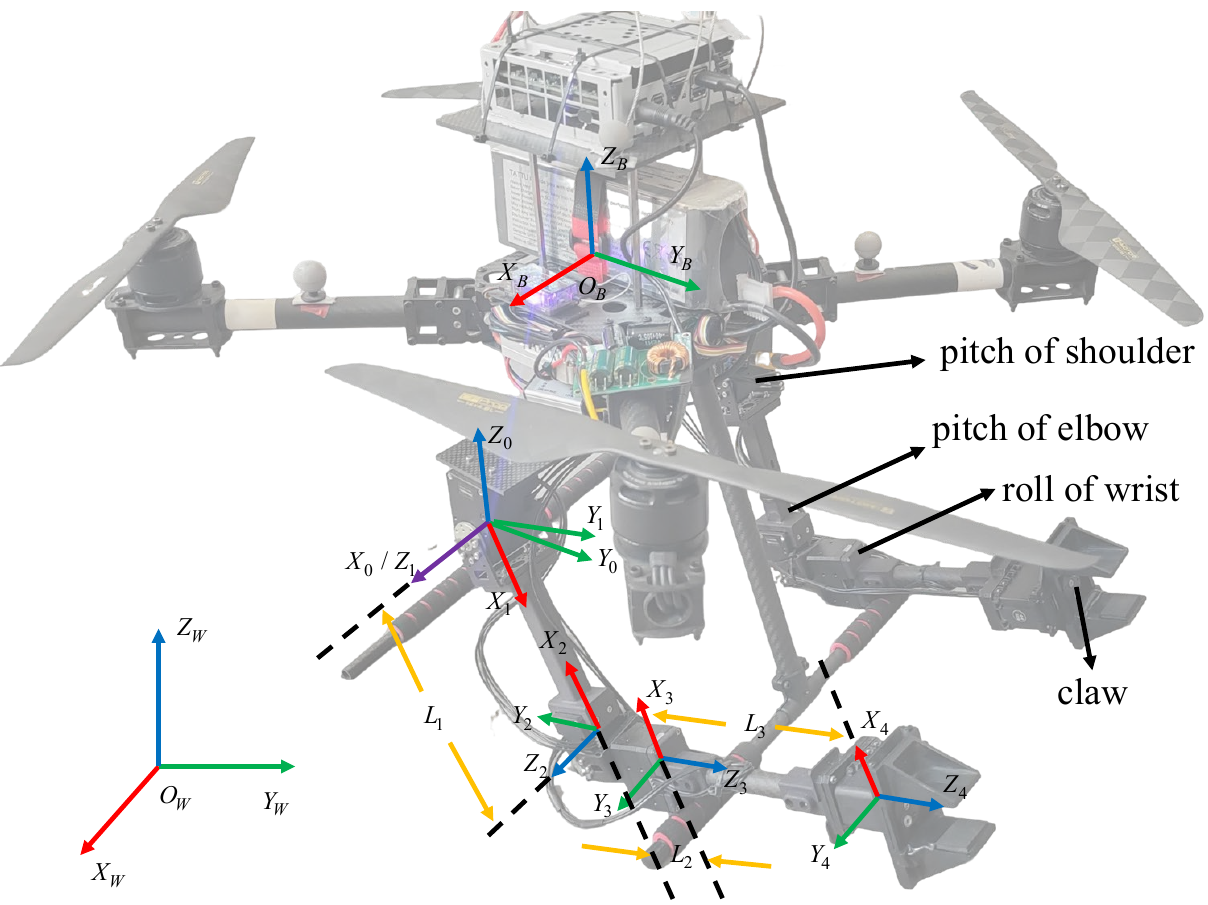}
\caption{Illustration of DAUAM system.}
\label{fig:dualarm}
\end{figure}

Fig.~\ref{fig:dualarm} illustrates the DAUAM considered in this work. Three types of coordinate frames are defined: the right-handed inertial frame $\mathcal{F}_W$, the multirotor body-fixed frame $\mathcal{F}_B$, and the joint frames $\mathcal{F}_i^j(X_i^j$-$Y_i^j$-$Z_i^j)$ of the dual-arm, where $i = 0,1,2,3,4$ indexes the links and $j = 1,2$ denote the right and left arm, respectively. The length of link $i$ is represented by $L_i$ $(i = 1,2,3)$. Each arm provides three actuated degrees of freedom: shoulder pitch, elbow pitch, and elbow roll, together with a claw mounted at the end-effector. The center of mass translational dynamics of the DAUAM is (\cite{model})
\begin{equation}
m_t \ddot {\bm p} + m_t g \bm e_3 + \bm F_m +\bm F_f(\dot{\bm p})+ \bm F_d = \bm U_c,
\end{equation}
where $m_t$ denotes the total mass of the system, 
$\bm p = [p_x, p_y, p_z]^\top \in \mathbb{R}^3$ is the position of the multirotor with respect to $\mathcal{F}_W$, 
$\bm e_3 = [0, 0, 1]^\top$ is the unit vector along the vertical axis, 
$\bm F_m(t) \in \mathbb{R}^3$ represents the  force exerted on the multirotor by the dual-arm, 
$\bm F_f(\dot{\bm p}) \in \mathbb{R}^3$ denotes frictional forces, 
$\bm F_d(t) \in \mathbb{R}^3$ collects general unknown nonlinear disturbances (e.g., unmodeled dynamics and parameter variations), 
$\bm U_c(t) = [U_{cx}, U_{cy}, U_{cz}]^\top \in \mathbb{R}^3$ is the control input generated by the multirotor.
Define 
\begin{align}
    &\bm u \triangleq \bm U_c/m_t - g \bm e_3,\\
    &\Delta(\bm p,\dot{\bm p},t) \triangleq -\left(\bm F_d + \bm F_m \right)/{m_t},\\
    & \bm f \triangleq -\bm F_f (\dot{\bm p})/{m_t},
\end{align}
yielding
\begin{align}\label{mod}
  \begin{cases}
  \dot{\bm p} &= \bm v,\\ 
  \dot{\bm v} &= \bm u + \Delta(\bm p,\bm v,t) + \bm f(\dot{\bm p}),
  \end{cases}
\end{align}
and the nonlinear function $\bm f(\dot{\bm p})$ will be approximated using a neural network based on the desired velocity 
$\dot{\bm p}_d$ in the following section.

\begin{assumption}
  There exist known bounds $\bar\Delta$ such that (\cite{model})
  \begin{align}
    \| \Delta(\bm p,\bm v,t)\| \le \bar\Delta.
  \end{align}
\end{assumption}

\begin{assumption}\label{ass:f-Lip}
The unknown nonlinear function $\bm f(\dot{\bm p})$ is locally Lipschitz with respect to $\dot{\bm p}$. 
That is, for any compact set $\Omega \subset \mathbb R^3$ containing the closed-loop trajectories, there exists a constant $L_f>0$ such that (\cite{Lip})
\begin{equation}\label{eq:f-Lip}
  \bigl\| f(\bm \nu_1) - f(\bm \nu_2) \bigr\|
  \le L_f \bigl\|\bm \nu_1 - \bm \nu_2 \bigr\|,\quad
  \forall\, \bm \nu_1,\bm \nu_2 \in \Omega.
\end{equation}
\end{assumption}

\begin{lemma}
  For $\bm s\in\mathbb R^3$, $p\in(0.5,1)$, $\epsilon>0$, design the smooth switch map:
  \begin{equation}
    \Theta(\bm s)=
    \begin{cases}
      \dfrac{\bm s}{\norm{\bm s}^{2(1-p)}}, & \norm{\bm s}^2>\epsilon,\\[6pt]
      \sin^2\!\big(\tfrac{\norm{\bm s}^2\pi}{2\epsilon}\big)\dfrac{\bm s}{\norm{\bm s}^{2(1-p)}}, & \norm{\bm s}^2\le\epsilon,
    \end{cases}
  \end{equation}
  which is continuous and satisfies that if $\ddot{\bm s}$ exists, then $\Theta(\bm s)$ is twice differentiable. In addition, there exist positive constants $k_{\Theta}, c_{\Theta}$, which make $\Theta(\bm s)$ satisfy the following inequality:
  \begin{align}\label{theta-bound}
    \| \Theta(\bm s) \| \le k_{\Theta} \| \bm s \| + c_{\Theta}.
  \end{align}
\end{lemma}

\begin{lemma}
Let $q_1,q_2\in\mathbb{R}$ and let $\nu_1,\nu_2,\nu_3>0$ be given constants, then the following inequality holds:
\begin{equation}
  \lvert q_1\rvert^{\nu_1}\lvert q_2\rvert^{\nu_2}
  \le
  \frac{\nu_1}{\nu_1+\nu_2}\,\nu_3\,\lvert q_1\rvert^{\nu_1+\nu_2}
  \;+\;
  \frac{\nu_2}{\nu_1+\nu_2}\,\nu_3^{-\frac{\nu_1}{\nu_2}}\,
  \lvert q_2\rvert^{\nu_1+\nu_2}.
\end{equation}
\end{lemma}

\begin{lemma}
Let $0<q\le 1$, $v_i\in\mathbb{R}$, and $w_i\ge 0$ for $i=1,\dots,n$, then the following inequalities hold:
\begin{equation}
  \left(\sum_{i=1}^{n}\lvert v_i\rvert\right)^{q}
  \le
  \sum_{i=1}^{n}\lvert v_i\rvert^{q},
  \;
  \left(\sum_{i=1}^{n} w_i\right)^{2}
  \le
  n \sum_{i=1}^{n} w_i^{2}.
\end{equation}
\end{lemma}

\section{Controller Design}\label{sec:controller}

\subsection{NN approximation}
In this paper, a three-layer neural network is employed. 
Let $\bm f_d := \bm f(\dot{\bm p}_d)$ and consider the input vector 
$\bm x_d = \big[\dot{\bm p}_d^\top,\,1\big]^\top \in \mathbb{R}^{N_0+1}$. 
By exploiting the universal approximation property of deep neural network (\cite{DNN_univ}), 
the mapping $\bm f_d(\bm x_d)$ and its neural-network approximation 
$\hat{\bm f}_d = [\hat f_{dx},\,\hat f_{dy},\,\hat f_{dz}]^\top$ can be represented in the following form:
\begin{align}
\label{DNN:fd}
\bm f_d(\bm x_d) &= V_1^\top \bm\phi \big( V_0^\top \bm x_d \big) + \bm \varepsilon(\bm x_d),\\
\label{DNN:fd:hat}
\hat{\bm f}_d(\bm x_d) &= \hat V_1^\top \bm\phi \big( \hat V_0^\top \bm x_d \big),
\end{align}
where $V_i,\hat V_i \in \mathbb{R}^{(N_i+1)\times N_{i+1}}$, $i=0,1$, denote the ideal constant weight matrices and their online estimates, respectively, $N_0$, $N_1$, $N_2$ are the numbers of neurons in the input, hidden, and output layers, the vector 
$\bm \phi = [\phi_{1},\dots,\phi_{N_1},\,1]^\top \in \mathbb{R}^{N_1+1}$ collects the hidden-layer activation functions, $\phi_j$ $(j=1,\dots,N_1)$ is the sigmoid activation at the $j$th hidden neuron and the term $\bm \varepsilon(\bm x_d) \in \mathbb{R}^{N_2}$ represents the bounded approximation error of the neural network.
 

The ideal weight matrices are assumed to be bounded, i.e.,
\begin{align}
  \label{weight_bound}
  \|V_i\|_F^2 = \operatorname{tr}(V_i^\top V_i) \le \bar V_i,\quad i = 0,1,
\end{align}
where $\|\cdot\|_F$ denotes the Frobenius norm of a matrix, $\operatorname{tr}(\cdot)$ denotes the trace operator, and $\bar V_i$ is a positive constants. 
Note that both the input vector $\bm x_d$ and the activation vector $\bm \phi$ are augmented with an additional element $1$ so that bias terms can be absorbed into the weight matrices.

For notational convenience, the following shorthand expressions are introduced:
\begin{align}
  \label{DNN:Phi}
  \bm \phi      = \bm \phi\!\left(V_0^\top \bm x_d\right),
    \hat{\bm \phi} = \bm \phi\!\left(\hat V_0^\top \bm x_d\right).
\end{align}
In addition, the weight-estimation errors are defined as
\begin{align}
  \label{DNN:Vtilde}
  \tilde V_i = V_i - \hat V_i,\quad i = 0,1.
\end{align}

\subsection{Dynamics of command filter}
Define the following change of coordinates:
\begin{equation}\label{eq:z-def}
\begin{cases}
\bm x_1 = \bm p - \bm p_d,\\
\bm x_2 = \bm v - \bm\chi,
\end{cases}
\end{equation}
and the corresponding command filter is given by
\begin{equation}\label{eq:cmd-filter}
\begin{cases}
\dot{\bm\chi} = \bm\xi,\\[2mm]
\dot{\bm\xi} = \dfrac{1}{\varepsilon^{2}}
\Big[-\zeta \arctan\!\big(\bm\chi-\bm\alpha_1\big)
      - \rho \arctan\!\big(\varepsilon\bm\xi\big)\Big],
\end{cases}
\end{equation}
with design parameters $\zeta>0$, $\rho>0$, and $\varepsilon>0$. Here $\bm\alpha_1$ serves as the input to the filter and will be specified as the virtual control and $\bm\chi\in\mathbb{R}^3$ is the filter output. The initial conditions are chosen as $\bm\chi(0)=\bm\alpha_1(0)$ and $\bm\xi(0)=\left[ 0, 0, 0 \right]^\top$. The compensated error is introduced as
\begin{equation}\label{eq:v-def}
  \bm y_i = \bm x_i - \bm\iota_i,\qquad i=1,2,
\end{equation}
where $\bm\iota_i\in\mathbb{R}^3$ are auxiliary compensation states. To counteract the filter-induced mismatch, the following compensation system is adopted:
\begin{equation}\label{eq:iota-dyn}
\left\{
\begin{aligned}
\dot{\bm\iota}_1 &= \left(\bm\iota_2+\bm\chi-\bm\alpha_1\right)
                  - c_1 \norm{\bm\iota_1}^{2}\bm\iota_1
                  - k \bm \iota_1
                  - k_1\Theta(\bm\iota_1),\\[2mm]
\dot{\bm\iota}_2 &= - \bm\iota_{1}
                  - c_2 \norm{\bm\iota_2}^{2}\bm\iota_2
                  - k_2\,\dfrac{\bm\iota_2}{\left(\bm\iota_2^{\top}\bm\iota_2\right)^{1-p}} + \hat{\bm f}_d(\bm x_d).
\end{aligned}
\right.
\end{equation}
Here $c_i,k_i,k \in \mathbb R^{3\times3}$ are positive definite diagonal matrices, and the exponent satisfies $\tfrac{1}{2}<p<1$. 

\begin{property}
  The command filter \eqref{eq:cmd-filter} is designed such that for any twice differentiable and bounded virtual control $\bm\alpha_1$, the filter tracking error is uniformly bounded (\cite{prop1}), i.e., there exists a positive constant vector $\bar{\bm\xi}=[\xi_1, \xi_2, \xi_3] \in \mathbb R^{3}$
  such that
  \begin{align}\label{chi-bound}
    \vert \bm\chi-\bm\alpha_{1}\vert\le \bar{\bm\xi}.
  \end{align}
\end{property}

\begin{property}
The chosen desired trajectory $\bm p_d \in \mathbb{C}^3$, and its first, second, and third derivatives are bounded.
\end{property}

\begin{property}
Based on \emph{Property 2}, one can conclude that $\bm \varepsilon(\bm x_d)$ is bounded, i.e.,
\begin{align}
\Vert \bm \varepsilon(\bm x_d) \Vert \leq \varepsilon_0,
\end{align}
where $\varepsilon_0$ is a positive constant.
\end{property}

\subsection{Controller design}
Based on the backstepping design philosophy, the control design is carried out in two steps.

\textit{Step~1.}
Using \eqref{mod}, \eqref{eq:z-def} and \eqref{eq:v-def}, one obtains
\begin{align}\label{y1-dot}
\dot{\bm y}_1
&= \bm v - \dot{\bm p}_d - \dot{\bm \iota}_1 \notag\\
&= \bm y_2 + \bm\chi + \bm\iota_2 - \dot{\bm p}_d - \dot{\bm \iota}_1.
\end{align}
Substituting the compensation dynamics \eqref{eq:iota-dyn} into \eqref{y1-dot} yields
\begin{equation}
\dot{\bm y}_1= \bm y_2 - \dot{\bm p}_d + \bm\alpha_1 + c_1 \norm{\bm\iota_1}^{2}\bm\iota_1 + k \bm \iota_1 + k_1 \Theta(\bm\iota_1).
\end{equation}
Choose the Lyapunov function candidate
\begin{equation}
V_1 = \frac{1}{2}\, \bm y_1^{\top}\bm y_1,
\end{equation}
then its time derivative is,
\begin{align}
\dot V_1
= \bm y_1^{\top}\!\left(\!
\bm y_2 \!-\! \dot{\bm p}_d \!+\! \bm\alpha_1 \!+\! c_1 \norm{\bm\iota_1}^{2}\bm\iota_1 \!+\! k\bm \iota_1 \!+\! k_1 \Theta(\bm\iota_1)
\right),
\end{align}
based on which, the virtual control $\bm\alpha_1$ is designed as follows:
\begin{align}\label{alpha1}
\bm\alpha_1 =& \dot{\bm p}_d - c_1 \norm{\bm\iota_1}^{2}\bm\iota_1 - k \bm \iota_1 - k_1 \Theta(\bm\iota_1) - \beta_{1}\norm{\bm y_1}^2 \bm y_1 \nonumber \\
&- \gamma_{1}\Theta(\bm y_1),
\end{align}
where $\beta_{1}, \gamma_{1} \in \mathbb R^{3\times3}$ are positive definite diagonal matrices. Consequently, one has
\begin{align}
\dot V_1
= \bm y_1^{\top}\!\Big(
\bm y_2 \!-\! \beta_{1}\norm{\bm y_1}^2 \bm y_1 - \gamma_{1}\Theta(\bm y_1)
\Big).
\end{align}

\textit{Step~2.}
Using \eqref{mod}, \eqref{eq:z-def}, \eqref{eq:v-def} and \eqref{eq:iota-dyn}, one obtains
\begin{align}\label{y2-dot}
\dot{\bm y}_2
=&\bm u + \Delta + \bm f(\dot{\bm p}) - \dot{\bm\chi} - \dot{\bm \iota}_2 \notag\\
=& \bm u + \Delta + \bm f(\dot{\bm p}) - \dot{\bm\chi} + \bm\iota_{1} + c_2 \norm{\bm\iota_2}^{2}\bm\iota_2 \nonumber \\
&+ k_2\,\dfrac{\bm\iota_2}{\big(\bm\iota_2^{\top}\bm\iota_2\big)^{1-p}} - \hat{\bm f}_d(\bm x_d) \nonumber \\
=& \bm u + \Delta - \dot{\bm\chi} + \bm\iota_{1} + c_2 \norm{\bm\iota_2}^{2}\bm\iota_2 + k_2\,\dfrac{\bm\iota_2}{\big(\bm\iota_2^{\top}\bm\iota_2\big)^{1-p}} \nonumber \\
&+ \tilde{\bm f} + \bm S(\dot{\bm p}, \dot{\bm p}_d),
\end{align}
where $\tilde{\bm f} = \bm f_d - \hat{\bm f}_d, \bm S(\dot{\bm p}, \dot{\bm p}_d) = \bm f(\dot{\bm p}) - \bm f(\dot{\bm p}_d)$. 
From Assumption~\ref{ass:f-Lip} and the definition of $\bm S(\dot{\bm p},\dot{\bm p}_d)$, one has
\begin{equation}\label{eq:S-Lip0}
  \bigl\| \bm S(\bm p_d,\dot{\bm p}_d) \bigr\|
  = \bigl\| \bm f(\dot{\bm p}) - \bm f(\dot{\bm p}_d) \bigr\|
  \le L_f \bigl\|\dot{\bm p} - \dot{\bm p}_d\bigr\|.
\end{equation}
Using $\dot{\bm p} = \bm v$, the coordinate transformations 
$\bm x_2 = \bm v - \bm \chi$, $\bm y_2 = \bm x_2 - \bm\iota_2$ and \eqref{theta-bound}, \eqref{chi-bound}, \eqref{alpha1}, it follows that
\begin{align}\label{eq:S-Lip1}
  &\left\|\dot{\bm p} - \dot{\bm p}_d \right\| \nonumber \\
  =& \left\|\bm x_2 + \bm \chi - \dot{\bm p}_d\right\| \nonumber\\
  =& \left\|\bm y_2 + \bm\iota_2 + (\bm\chi -\bm\alpha_{1}) + (\bm\alpha_{1} - \dot{\bm p}_d)\right\| \nonumber \\
  \le& \|\bm y_2\| + \|\bm\iota_2\| + \Vert\bar{\bm\xi}\Vert + \Vert c_1\Vert_2 \|\bm\iota_1\|^3 + \Vert k+k_1k_\Theta \Vert_2 \|\bm\iota_1\| \nonumber \\
  & + \Vert \beta_1\Vert_2 \|\bm y_1\|^3 + \Vert\gamma_1k_\Theta\Vert_2 \|\bm y_1\| + \Vert k_1c_\Theta + \gamma_1 c_{\Theta}\Vert_2 \nonumber \\
  =& \|\bm y_2\| + \|\bm\iota_2\| + c_1\|\bm\iota_1\|^3 + \beta_1\|\bm y_1\|^3 + a_{\iota1}\|\bm\iota_1\| \nonumber \\
  & + a_{y1}\|\bm y_1\| + c_{s1},
\end{align}
where $a_{\iota1} = \Vert k+k_1k_\Theta \Vert_2, a_{y1}=\Vert\gamma_1k_\Theta\Vert_2, c_{s1}=\Vert\bar{\bm\xi}\Vert+\Vert k_1c_\Theta+\gamma_1c_{\Theta}\Vert_2$. Thus, one has
\begin{align}\label{eq:S-Lip2}
  \bigl\| \bm S(\dot{\bm p},\dot{\bm p}_d) \bigr\|
  \le L_f&\left(\| \bm y_2\| \!+\! \|\bm\iota_2\| \!+ \!\Vert c_1 \Vert_2\|\bm\iota_1\|^3 + \Vert \beta_1 \Vert_2 \|\bm y_1\|^3 \right. \nonumber\\
  & \left.+ a_{\iota1}\|\bm\iota_1\| + a_{y1}\|\bm y_1\| + c_{s1} \right).
\end{align}

Furthermore, one obtains
\begin{align}
  \bm y_2^\top \bm S
  \le& L_f\|\bm y_2\|\left(\| \bm y_2\| + \|\bm\iota_2\| + \Vert c_1 \Vert_2 \|\bm\iota_1\|^3 + \Vert\beta_1\Vert_2 \|\bm y_1\|^3 \right. \nonumber\\
  & \quad\quad\quad\quad \left.+ a_{\iota1}\|\bm\iota_1\| + a_{y1}\|\bm y_1\| + c_{s1} \right) \notag\\
  =& L_f\|\bm y_2\|^2 \!+\! L_f\|\bm y_2\| \|\bm\iota_2\|
     \!+\! L_f \Vert c_{1} \Vert_2 \|\bm y_2\| \|\bm \iota_1\|^3 \nonumber\\
    & + L_f \Vert \beta_1 \Vert_2 \|\bm y_2\| \|\bm y_1\|^3 + L_f a_{\iota1}\|\bm y_2\| \|\bm \iota_1\| \nonumber\\
    & + L_f a_{y1}\|\bm y_2\| \|\bm y_1\| + L_f c_{s1}\|\bm y_2\|. \label{eq:y2S-1}
\end{align}
For any design constants $\eta_1,\eta_2,\eta_3,\eta_4>0$, Young's inequality yields
\begin{align}
  &L_f\|\bm y_2\|\,\|\bm\iota_2\|
  \le \frac{\eta_1}{2}\|\bm y_2\|^2 + \frac{L_f^2}{2\eta_1}\|\bm\iota_2\|^2, \label{eq:y2S-Young1}\\
  &L_f c_{s1}\|\bm y_2\|
  \le \frac{\eta_2}{2}\|\bm y_2\|^2 + \frac{L_f^2 c_{s1}^2}{2\eta_2}, \\
  &L_f a_{\iota_1}\|\bm y_2\|\|\bm \iota_1\|
  \le \frac{\eta_3}{2}\|\bm y_2\|^2+\frac{L_f^2 a_{\iota_1}^2}{2\eta_3}\|\bm \iota_1\|^2, \\
  &L_f a_{y_1}\|\bm y_2\|\|\bm y_1\|
  \le \frac{\eta_4}{2}\|\bm y_2\|^2+\frac{L_f^2 a_{y_1}^2}{2\eta_4}\|\bm y_1\|^2, 
\end{align}
\begin{align}
  &L_f \Vert c_1 \Vert_2\|\bm y_2\|\|\bm \iota_1\|^3
  \le \sigma^{(1)}_{y_2,4}\|\bm y_2\|^4+\sigma_{\iota_1,4}\|\bm \iota_1\|^4, \\
  &L_f \Vert\beta_1\Vert_2\|\bm y_2\|\|\bm y_1\|^3
  \le \sigma^{(2)}_{y_2,4}\|\bm y_2\|^4+\sigma_{y_1,4}\|\bm y_1\|^4, \label{eq:y2S-Young2}
\end{align}
where $\sigma^{(1)}_{y_2,4}, \sigma_{\iota_1,4}, \sigma^{(2)}_{y_2,4}, \sigma_{y_1,4}$ are known constants obtained by Lemma 2. Substituting \eqref{eq:y2S-Young1}--\eqref{eq:y2S-Young2} into \eqref{eq:y2S-1} gives
\begin{align}\label{eq:y2S-final}
&\bm y_2^\top \bm S
\le \bar\sigma_{1}\|\bm y_2\|^4
   + \bar\sigma_{2}\|\bm y_1\|^4
   + \bar\sigma_{3}\|\bm \iota_1\|^4
   + \bar\sigma_{4}\|\bm y_2\|^2 \notag\\
&\quad\quad\quad + \bar\sigma_{5}\|\bm y_1\|^2
   + \bar\sigma_{6}\|\bm \iota_1\|^2
   + \bar\sigma_{7}\|\bm \iota_2\|^2
   + c_s, \\
&\bar\sigma_{1} = \sigma^{(1)}_{y_2,4} + \sigma^{(2)}_{y_2,4}, 
\bar\sigma_{2} = \sigma^{(1)}_{y_1,4}, 
\bar\sigma_{3} = \sigma_{\iota_1,4}, \nonumber \\
&\bar\sigma_{4} = L_f + \frac{\eta_1+\eta_2+\eta_3+\eta_4}{2}, \nonumber \\
&\bar\sigma_{5}=\frac{L_f^2 a_{y_1}^2}{2\eta_4}, \bar\sigma_{6}=\frac{L_f^2 a_{\iota_1}^2}{2\eta_3}, \bar\sigma_{7} = \frac{L_f^2}{2\eta_1}, c_s = \frac{L_f^2 c_{s1}^2}{2\eta_2}. \nonumber
\end{align}

Choose the Lyapunov function candidate
\begin{equation}
V_2 = V_1 + \frac{1}{2}\, \bm y_2^{\top}\bm y_2 +  \frac12 \mathrm{tr}(\tilde V_1^\top \Gamma_1^{-1} \tilde V_1) + \frac12 \mathrm{tr}(\tilde V_0^\top \Gamma_0^{-1} \tilde V_0),
\end{equation}
then its time derivative is,
\begin{align}\label{V2-dot}
\dot V_2
= &\bm y_1^{\top}\!\left(\bm y_2 \!-\! \beta_{1}\norm{\bm y_1}^2 \bm y_1 - \gamma_{1}\Theta(\bm y_1)\right) + \bm y_2^{\top}\!\left(
\bm u \!+\! \Delta \!-\! \dot{\bm\chi} \!+\! \bm\iota_{1} \right.\nonumber\\
& \left. + c_2 \norm{\bm\iota_2}^{2}\bm\iota_2 \!+\! k_2\,\dfrac{\bm\iota_2}{\left(\bm\iota_2^{\top}\bm\iota_2\right)^{1-p}}\! + \!\tilde{\bm f} \!+ \! \bm S(\dot{\bm p}, \dot{\bm p}_d) \right) \nonumber \\
& + \mathrm{tr}(\tilde V_1^\top \Gamma_1^{-1} \dot {\tilde V}_1) + \mathrm{tr}(\tilde V_0^\top \Gamma_0^{-1} \dot {\tilde V}_0).
\end{align}
Design the auxiliary virtual control input
\begin{align}
\bm\alpha_2 =&\, \dot{\bm \chi} - \bm x_1 - c_2 \norm{\bm\iota_2}^{2} \bm\iota_2 - k_2\,\dfrac{\bm\iota_2}{\big(\bm\iota_2^{\top}\bm\iota_2\big)^{1-p}} - \beta_{2}\norm{\bm y_2}^2 \bm y_2 \nonumber \\
&- \gamma_{2}\dfrac{\bm y_2}{\big(\bm y_2^{\top}\bm y_2\big)^{1-p}},
\end{align}
where $\beta_{2}, \gamma_{2} \in \mathbb R^{3\times3}$ are positive definite diagonal matrices. The neural network weight update laws are designed as
\begin{align}
  \label{update1}
& \dot{\hat V}_1 = \mathrm{proj}\!\left(\Gamma_1 \bm\phi(\hat V_0^\top \bm x_d)\, \bm y_2^\top \right), \\
\label{update2}
& \dot{\hat V}_0 = \mathrm{proj}\!\left(\Gamma_0 \bm x_d \big[ \phi^{\prime\top}\!(\hat V_0^\top \bm x_d) \hat V_1 \bm y_2 \big]^\top \right).
\end{align}
where $\Gamma_0\in \mathbb R^{(N_0+1)\times(N_0+1)}, \Gamma_1\in \mathbb R^{(N_1+1)\times(N_1+1)}$ are positive definite diagonal gain matrices, $\phi^\prime\in \mathbb R^{(N_1+1)\times N_1}$ denotes the gradient of the activation function and $\phi_i^\prime (\bm a) = \frac{\partial}{\partial \bm b} \bm \phi(\bm b)\big\vert_{\bm b = \bm a}, \forall \bm a \in \mathbb{R}^{N_1}$, $\mathrm{proj}(\cdot)$ denotes the standard smooth projection operator (\cite{proj}), which keeps the estimates $\dot{\hat V}_i$ bounded. 

By utilizing the Mean Value Theorem, there exists a vector $\bm \zeta$ between $V_0^\top \bm x_d$ and $\hat V_0^\top \bm x_d$ such that $ \bm\phi -  \hat{\bm\phi} = \phi^\prime(\bm \zeta) \tilde V_0^\top \bm x_d $. Furthermore, one has
\begin{align}\label{tilde-f}
\tilde{\bm f} =& \bm f_d - \hat{\bm f}_d \notag\\
=& V_1^\top \bm\phi + \bm \varepsilon - \hat V_1^\top \hat{\bm\phi}  \notag\\
=& V_1^\top \bm\phi - V_1^\top \hat{\bm\phi} + V_1^\top \hat{\bm\phi} - \hat V_1^\top \hat{\bm\phi} + \bm \varepsilon \notag \\
=& \hat V_1^\top \! \phi^\prime(\bm \zeta) \! \tilde V_0^\top \bm x_d \!+\! \tilde V_1^\top \!\phi^\prime(\bm \zeta)\!  \tilde V_0^\top \bm x_d \!+\! \tilde V_1^\top \hat{\bm\phi} \!+\! \bm \varepsilon \notag\\
=&\hat V_1^\top \phi^\prime(\hat V_0^\top \bm x_d) \tilde V_0^\top \bm x_d \!+\! \tilde V_1^\top \!\phi^\prime(\bm \zeta)  \tilde V_0^\top \bm x_d \notag\\
& \hat V_1^\top \!\!\left[ \phi^\prime(\bm \zeta) - \phi^\prime(\hat V_0^\top \bm x_d) \right] \! \tilde V_0^\top \bm x_d \!+\! \tilde V_1^\top \hat{\bm\phi} \!+\! \bm \varepsilon.
\end{align}
Substituting \eqref{update1}, \eqref{update2}, \eqref{tilde-f} into \eqref{V2-dot}, one obtains
\begin{align}
  \label{dotV2}
&\dot V_2
= -\Vert\beta_{1}\Vert_2 \left(\bm y_1^\top \bm y_1 \right)^2- \Vert\gamma_{1}\Vert_2 \bm y_1^{\top} \Theta(\bm y_1) - \Vert\beta_{2}\Vert_2 \left(\bm y_2^\top \bm y_2 \right)^2 \nonumber\\
&\quad\quad - \Vert\gamma_{2}\Vert_2 {\big(\bm y_2^{\top}\bm y_2\big)^{p}} + \bm y_2^\top \left( \bm u \!-\! \bm\alpha_2\right) + \bm y_2^\top\Delta + N_B + \bm y_2^\top \bm S , \nonumber\\
&N_B = N_{B1} + N_{B2} + N_{B3}, \nonumber \\
&N_{B1} = \bm y_2^\top \tilde V_1^\top \phi^\prime(\bm \zeta)  \tilde V_0^\top \bm x_d, N_{B3} = \bm y_2^\top \bm \varepsilon, \nonumber \\
&N_{B2} = \bm y_2^\top \hat V_1^\top \!\!\left[ \phi^\prime(\bm \zeta) \!- \phi^\prime(\hat V_0^\top \bm x_d) \right] \! \tilde V_0^\top\! \bm x_d,
\end{align}
and based on the boundedness of $\hat V_1$, $\hat V_0$, $\phi^\prime(\cdot)$, and $\bm x_d$, there exist known positive constants $a$, $l_1$, $l_2$, and $\bar\varepsilon$ such that the following inequalities hold:
\begin{align}\label{NB-bound}
|N_{B1}| &\le \tfrac{a}{2}(\|\tilde V_1\|^2+\|\tilde V_0\|^2)
              + \tfrac{l_1^2}{2a}\|\bm x_d\|^2\|\bm y_2\|^2,\nonumber \\
|N_{B2}| &\le \tfrac{a}{2}\|\tilde V_0\|^2
              + \tfrac{l_2^2}{2a}\|\bm x_d\|^2\|\bm y_2\|^2, \nonumber \\
|N_{B3}| &\le \tfrac12\|\bm y_2\|^2+\tfrac12 \varepsilon_0^2.
\end{align}
As for $\bm y_2^\top \Delta$, by utilizing H\"{o}lder inequality, it can be rearranged as follows:
\begin{align}\label{y2delta}
  \bm y_2^\top \Delta \le \frac{\Vert\beta_2\Vert_2}{4}\|\bm y_2 \|^4 + \frac{3}{4}\Vert\beta_2\Vert_2^{-\frac13} \bar\Delta^{\frac43}.
\end{align}

Furthermore, define the event-triggered control input as
\begin{equation}\label{eq:ETC-law}
\left\{
\begin{aligned}
\bar{\bm u}(t) &= \bm\alpha_2(t) - \Big(\sigma+\tfrac{\kappa(t)}{\delta}\Big)\operatorname{sgn}\!\big(\bm y_2(t)\big), \\[2mm]
\bm u(t) &= \bar{\bm u}(t_i), \;\;  t\in[t_i,\,t_{i+1}), \\[1mm]
t_{i+1} &= \inf\left\{ t>t_i \middle| \kappa(t) + \delta\big(\sigma - |\tilde e_j(t)|\big) < 0 \right. \\
& \quad\quad\quad \left. \exists \, j\!=\!1,2,3 \right\}, \\[1mm]
\dot\kappa(t) &= -\,\kappa(t) + \max_{1\le j\le 3}\!\left(\sigma - |\tilde e_j(t)|\right),
\end{aligned}\right.
\end{equation}
where $\sigma>0$, $\delta>0$, and $\kappa(0)\ge0$ are design constants,
$\operatorname{sgn}(\cdot)$ acts componentwise,
$\tilde{\bm e}(t)=\left[ \tilde e_1, \tilde e_2, \tilde e_3 \right]^\top\triangleq \bm u(t)-\bar{\bm u}(t)$ is the hold error.

\section{Stability Analysis}\label{sec:stability}
\begin{theorem}\label{thm:ETC}
Consider system~\eqref{mod}, with the event-triggered implementation in~\eqref{eq:ETC-law}, all closed-loop
signals are bounded and $\bm p(t)$ tracks $\bm p_d(t)$ within a fixed-time practical bound under the condition that $ \Vert c_1 \Vert_2> \frac43\bar\sigma_{3}, \Vert \beta_1 \Vert_2> \frac43\bar\sigma_{2}, \Vert \beta_{2} \Vert_2$ $> 2\bar\sigma_1 $. Moreover, Zeno phenomena do not occur.
\end{theorem}

\begin{proof}

\medskip\noindent\textit{Exclusion of Zeno.}
Assume, by contradiction, that the inter-event interval $\Delta_i=t_{i+1}-t_i\to 0$.
From~\eqref{eq:ETC-law}, and based on the fact that $\bm\alpha_2(\cdot)$ and $\kappa(\cdot)$ are continuous, one has
\begin{align}
  \label{zeno}
&\lim_{\Delta_i\to 0}\!\left\|\bar{\bm u}(t_{i+1})-\bar{\bm u}(t_i)\right\| \notag\\
 =& \lim_{\Delta_i\to 0}\!\left\|
 \bm\alpha_2(t_{i+1})-\bm\alpha_2(t_i)
 -\left(\sigma+\frac{\kappa(t_{i+1})}{\delta}\right)\! \operatorname{sgn}(\bm y_2(t_{i+1})) \right. \notag\\
 &\left. +\left(\sigma+\frac{\kappa(t_{i})}{\delta}\right) \operatorname{sgn}(\bm y_2(t_i)) \right\| \notag\\
 =&\lim_{\Delta_i\to 0}\! \left\|\sigma \left[ \operatorname{sgn}\left(\bm y_2 \left(t_{i}\right) \right) - \operatorname{sgn}\left(\bm y_2\left(t_{i+1}\right)\right) \right] \right\| \notag\\
 &+ \lim_{\Delta_i\to 0} \left\| \frac{\kappa(t_{i})}{\delta} \! \operatorname{sgn}(\bm y_2(t_{i+1})) - \frac{\kappa(t_{i+1})}{\delta} \! \operatorname{sgn}(\bm y_2(t_{i+1})) \right\|,\notag\\
 &+ \lim_{\Delta_i\to 0} \left\| \frac{\kappa(t_{i})}{\delta} \! \operatorname{sgn}(\bm y_2(t_{i})) - \frac{\kappa(t_{i})}{\delta} \! \operatorname{sgn}(\bm y_2(t_{i+1})) \right\| \notag\\
 =&\lim_{\Delta_i\to 0}\! \left\|\sigma \left[ \operatorname{sgn}\left(\bm y_2 \left(t_{i}\right) \right) - \operatorname{sgn}\left(\bm y_2\left(t_{i+1}\right)\right) \right] \right\| \notag\\
 &+ \lim_{\Delta_i\to 0} \left\| \frac{\kappa(t_{i})}{\delta} \! \operatorname{sgn}(\bm y_2(t_{i})) - \frac{\kappa(t_{i})}{\delta} \! \operatorname{sgn}(\bm y_2(t_{i+1})) \right\|.
\end{align}
If no zero crossing occurs for $\bm y_2$, \eqref{zeno} tends to zero.
If a zero crossing occurs, then with $\operatorname{sgn}(0)=0$,
$\|\operatorname{sgn}(\bm y_2(t_{i+1}))-\operatorname{sgn}(\bm y_2(t_i))\|\le 1$, and thus
\begin{equation}
  \label{zeno1}
\lim_{\Delta_i\to 0}\!\left\|\bar{\bm u}(t_{i+1})-\bar{\bm u}(t_i)\right\|
\le \sigma+\frac{\kappa(t_i)}{\delta}. 
\end{equation}
Therefore, \eqref{zeno1} always satisfies, which contradicts the triggered condition. Hence a strictly positive dwell time exists and Zeno behavior is excluded.

From~\eqref{eq:ETC-law}, for each inter-event interval $t\in[t_i,$ $t_{i+1})$ there exists
a time-varying vector $\bm\tau(t) =\left[ \tau_1, \tau_2, \tau_3\right]^\top$ $\in\mathbb{R}^3$ such that $\| \bm\tau(t)\|_\infty\le 1$,
$\tau_j(t_i)=0 \, (j=1,2,3)$, and
\[
\bar{\bm u}(t) = \bm u(t) + \Big(\sigma+\frac{\kappa(t)}{\delta}\Big) \bm\tau(t),
\]
where $\bar{\bm u}$ is the continuously updated control before the zero-order hold.
Hence
\begin{align}
\bm y_2^\top \left(\bm u-\bm\alpha_2\right)
 &= \bm y_2^\top \!\left(\bar{\bm u} - \left(\sigma+\frac{\kappa}{\delta}\right)\bm\tau - \bm\alpha_2\right)\notag\\
 &= \bm y_2^\top \!\left( - \left(\sigma+\tfrac{\kappa}{\delta}\right)\operatorname{sgn}\!\left(\bm y_2\right) - \left(\sigma+\frac{\kappa}{\delta}\right)\bm\tau \right)\notag\\
 &= \bm y_2^\top \!\left(\!-\,\operatorname{sgn}(\bm y_2)-\bm\tau\right)\left(\sigma+\frac{\kappa}{\delta}\right)
 \;\le\; 0 ,
\end{align}
which yields
\begin{equation}
\bm y_2^\top \left(\bm u-\bm\alpha_2\right)\le 0 .
\label{eq:key-ineq}
\end{equation}

Using~\eqref{eq:key-ineq} in the Lyapunov derivative obtained from~\eqref{dotV2}, and utilizing \eqref{eq:y2S-final}, \eqref{NB-bound}, \eqref{y2delta}, one can conclude that 

\begin{align}
  \dot V_2 
  &\le -\Vert \beta_1 \Vert_2(\bm y_1^\top\bm y_1)^2 - \Vert \gamma_1 \Vert_2 \bm y_1^\top\Theta(\bm y_1)
        - \Vert \beta_2 \Vert_2(\bm y_2^\top\bm y_2)^2 \notag\\
  &\quad - \Vert\gamma_2\Vert_2(\bm y_2^\top\bm y_2)^p + \frac{\Vert\beta_2\Vert_2}{4}\|\bm y_2 \|^4 + \frac{3}{4}\Vert\beta_2\Vert_2 ^{-\frac13} \bar\Delta^{\frac43}  \notag\\
  &\quad + c_{N1}\|\bm y_2\|^2 + c_{N0} + \bar\sigma_{1}\|\bm y_2\|^4 + \bar\sigma_{2}\|\bm y_1\|^4
   + \bar\sigma_{3}\|\bm \iota_1\|^4 \nonumber\\
  &\quad + \bar\sigma_{4}\|\bm y_2\|^2 + \bar\sigma_{5}\|\bm y_1\|^2
   + \bar\sigma_{6}\|\bm \iota_1\|^2 + \bar\sigma_{7}\|\bm \iota_2\|^2
   + c_s,
\end{align}
where $c_{N0},c_{N1}>0$ are constants arising from the bounds of $N_{B1}, N_{B2}$ and $N_{B3}$. Rearranging terms yields
\begin{align}\label{eq:Vn-dot}
  \dot V_2 
  &\le -\Vert \beta_1\Vert_2(\bm y_1^\top\bm y_1)^2 - \Vert\gamma_1\Vert_2 \bm y_1^\top\Theta(\bm y_1)
        - \Vert\beta_2\Vert_2(\bm y_2^\top\bm y_2)^2 \notag\\
  &\quad - \Vert\gamma_2\Vert_2(\bm y_2^\top\bm y_2)^p + \left(\frac{\Vert\beta_2\Vert_2}{4}+\bar\sigma_{1}\right)\!\! \|\bm y_2 \|^4 + \bar\sigma_{2}\|\bm y_1\|^4 \notag\\
  & \quad + \bar\sigma_{3}\|\bm \iota_1\|^4  + \bar c_{y2}\|\bm y_2\|^2  + \bar\sigma_{5}\|\bm y_1\|^2
   + \bar\sigma_{6}\|\bm \iota_1\|^2 \nonumber \\
   &\quad  + \bar\sigma_{7}\|\bm \iota_2\|^2
   + \bar c_0, 
\end{align}
with $\bar c_{y2}=c_{N1}+\bar\sigma_{4}$ and $\bar c_0 = \frac{3}{4}\Vert\beta_2\Vert_2^{-\frac13} \bar\Delta^{\frac43} + c_{N0}+c_s$. By Young's inequality
\begin{align}\label{y-iota}
  &\bar c_{y2}\|\bm y_2\|^2
\le \frac{\Vert\beta_2\Vert_2}{4}\|\bm y_2\|^4 + \frac{\bar c_{y2}^2}{\Vert \beta_2 \Vert_2}, \nonumber\\
&\bar \sigma_{5}\|\bm y_1\|^2
\le \frac{\Vert\beta_1\Vert_2}{4}\|\bm y_1\|^4 + \frac{\bar \sigma_{5}^2}{\Vert \beta_1 \Vert_2}, \nonumber
\end{align}
\begin{align}
&\bar \sigma_{6}\|\bm\iota_1\|^2
\le \frac{\Vert c_1 \Vert_2}{4}\|\bm\iota_1\|^4 + \frac{\bar \sigma_{6}^2}{\Vert c_1 \Vert_2},\nonumber\\
&\bar \sigma_{7}\|\bm\iota_2\|^2
\le \frac{\Vert c_2 \Vert_2}{4}\|\bm\iota_2\|^4 + \frac{\bar \sigma_{7}^2}{\Vert c_2 \Vert_2}.
\end{align}

Then, construct the Lyapunov function candidate $V$ as
\begin{equation}\label{eq:V-comp}
V = V_2 + \frac{1}{2}\sum_{i=1}^{2}\bm\iota_i^{\top}\bm\iota_i,
\end{equation}
\noindent
its time derivative is calaulated as follows
\begin{equation}
  \label{dot-V}
\dot V = \dot V_2 + \sum_{i=1}^{2} \bm\iota_i^{\top}\dot{\bm\iota}_i .
\end{equation}

\noindent
Using the compensator dynamics,
one obtains
\begin{align}
  \label{V-iota}
\sum_{i=1}^{2}\bm\iota_i^{\top}\dot{\bm\iota}_i
&\le \!
- \Vert c_1 \Vert_2\left( \bm\iota_1^{\top}\bm\iota_1\right)^{\!2}
\!-\! \Vert h  \Vert_2 \bm\iota_1^{\top} \bm\iota_1
\!-\! \Vert k_1 \Vert_2 \bm\iota_1^{\top}\Theta(\bm\iota_1)
 \nonumber \\
& \quad 
\!-\! \Vert c_2 \Vert_2(\bm\iota_2^{\top}\bm\iota_2)^{2}
\!-\! \Vert k_2 \Vert_2(\bm\iota_2^{\top}\bm\iota_2)^{p}
\!+\! \bm\iota_2^{\top} \hat{\bm f}_d(\bm x_d) \nonumber\\
& \quad \!+\! \left|\bm\iota_1^{\top}\bar{\bm\xi}\right|.
\end{align}
Based on the boundness of $\bm x_d, \hat{V}_i$, and $\bm\phi(\cdot)$, there exists a positive constant \(\bar{f}_d\) such that
\begin{align}
\left|\hat{\bm f}_d(\bm x_d)\right| \le \bar{f}_d,
\end{align}
and by Young's inequality,
\begin{align}
  &\left|\bm\iota_1^{\top}\bar{\bm\xi}\right| \le \| h \|_2 \bm\iota_1^\top \bm\iota_1 + \frac{1}{4 \|h\|_2}\bar{\bm\xi}^\top \bar{\bm\xi}, \\
  &\left|\bm\iota_2^{\top} \hat{\bm f}_d\right| \le \frac{\|c_2\|_2}{2}\|\bm\iota_2\|^4 + c_f,
\end{align}
where $c_f = \frac{3}{4} (2 \|c_2\|_2)^{-\frac13} \bar{f}_d^{\frac43}$. Then, \eqref{V-iota} can be further rearranged as 
\begin{align}
  \label{V-iota2}
&\sum_{i=1}^{2}\bm\iota_i^{\top}\dot{\bm\iota}_i  \nonumber\\
\le&
- \|c_1\|_2 \left( \bm\iota_1^{\top}\bm\iota_1\right)^{\!2}
- \| k_1 \|_2 \bm\iota_1^{\top}\Theta(\bm\iota_1) - \frac{\|c_2\|_2}{2}(\bm\iota_2^{\top}\bm\iota_2)^{2} \nonumber\\
& - \|k_2\|_2(\bm\iota_2^{\top}\bm\iota_2)^{p} + \frac{1}{4 \|h\|_2}\bar{\bm\xi}^\top \bar{\bm\xi} + c_f.
\end{align}
Substituting \eqref{eq:Vn-dot}, \eqref{y-iota} and \eqref{V-iota2} into \eqref{dot-V} yields
\begin{align}
\dot V
\leq& \!-\! \|\beta_{1}\|_2 \left(\bm y_1^\top \bm y_1 \right)^2\!-\! \|\gamma_{1}\|_2 \bm y_1^{\top} \Theta(\bm y_1) \!-\! \|\beta_{2}\|_2 \left(\bm y_2^\top \bm y_2 \right)^2  \nonumber\\
& \!-\! \|\gamma_{2}\|_2{\big(\bm y_2^{\top}\bm y_2\big)^{p}} \!\!+\!\! \left(\!\frac{\|\beta_2\|_2}{2} \!\!+\!\! \bar\sigma_{1}\!\right)\! \|\bm y_2\|^4 \!\!+\!\! \left(\! \frac{\|\beta_1\|}{4} \!\!+\!\! \bar\sigma_{2} \!\right) \!\|\bm y_1\|^4  \nonumber \\
&\!+\! \bar\sigma_{3}\|\bm \iota_1\|^4  \!+\! \frac{\|c_1\|_2}{4}\|\bm \iota_1\|^4
\!+\! \frac{\|c_2\|_2}{4}\|\bm \iota_2\|^4
\!-\! \|c_1\|_2 \!\left(\bm\iota_1^{\top}\bm\iota_1\right)^{\!2}  \nonumber\\
& \!-\! \|k_1\|_2\bm\iota_1^{\top}\Theta(\bm\iota_1)
-\frac{\|c_2\|_2}{2}(\bm\iota_2^{\top}\bm\iota_2)^{2} \!-\! \|k_2\|_2(\bm\iota_2^{\top}\bm\iota_2)^{p} \nonumber\\
&\!+\! \frac{\bar{\bm\xi}^\top \bar{\bm\xi}}{4 \|h\|_2} \!+\! c_f + \frac{\bar c_{y2}^2}{\beta_2}+ \frac{\bar \sigma_{5}^2}{\beta_1} + \frac{\bar \sigma_{6}^2}{c_1} + \frac{\bar \sigma_{7}^2}{c_2}\!+\! \bar c_0.
\label{eq:V-dot}
\end{align}

Next, the following two cases should be taken into account when dealing with the switch function.
\paragraph*{Case 1:}
\(\bm\iota_1^{\top}\bm\iota_1>\epsilon\) and \(\bm y_1^{\top}\bm y_1>\epsilon\).
In this region,
\[
\Theta(\bm s)=\frac{\bm s}{(\bm s^{\top}\bm s)^{1-p}} .
\]
Substituting \(\Theta(\cdot)\) into the bound of \(\dot V\) gives
\begin{align}
\dot V
\leq& - \sum_{i=1}^{2}\! \left(\! \|\beta_{i}\|_2\! \left(\bm y_i^\top \bm y_i \right)^2 \!\!+\! \|\gamma_{i}\|_2\! \left( \bm y_i^{\top} \bm y_i \right)^p \!\!+\! \|c_i\|_2 \!\left(\bm\iota_i^{\top}\bm\iota_i\right)^{2} \!\! \right. \nonumber\\
&\left. + k_i\!\left(\bm\iota_i^{\top}\bm\iota_i\right)^{p}\right) \!\!+\!\! \left(\!\frac{\|\beta_2\|_2}{2}\!\!+\!\!\bar\sigma_{1}\!\right)\!\|\bm y_2 \|^4 \!\!+\!\! \left( \!\frac{\|\beta_1\|_2}{4}\! \!+\! \!\bar\sigma_{2}\!\right)\! \|\bm y_1\|^4  \nonumber\\
& +\left(\! \frac{\|c_1\|_2}{4}\!+\!\bar\sigma_{3}\!\right)\|\bm \iota_1\|^4 \!+\! \frac{3 \|c_2\|_2}{4}\|\bm \iota_2 \|^4 \!+\! \frac{\bar{\bm\xi}^\top \bar{\bm\xi}}{4 \|h\|_2} \!+\! c_f \nonumber\\
& + \frac{\bar c_{y2}^2}{\beta_2}+ \frac{\bar \sigma_{5}^2}{\beta_1}+ \frac{\bar \sigma_{6}^2}{c_1} + \frac{\bar \sigma_{7}^2}{c_2}\!+\! \bar c_0.
\end{align}

\paragraph*{Case 2:}
\(\bm\iota_1^{\top}\bm\iota_1 \le \epsilon\) or \(\bm y_1^{\top}\bm y_1 \le \epsilon\). Let $a=1 \,(-1)$ represents $\bm\iota_1^{\top}\bm\iota_1 \le \epsilon \,(\bm\iota_1^{\top}\bm\iota_1 > \epsilon)$ and $b=1\,(-1)$ represents $\bm y_1^{\top}\bm y_1 \le \epsilon \,(\bm y_1^{\top}\bm y_1 > \epsilon)$. Use the smooth switch
\[
\Theta(\bm s)=
\begin{cases}
\dfrac{\bm s}{(\bm s^{\top}\bm s)^{1-p}}, & \bm s^{\top}\bm s>\epsilon,\\[6pt]
\dfrac{\bm s}{(\bm s^{\top}\bm s)^{1-p}}
-\cos^{2}\!\Big(\dfrac{\bm s^{\top}\bm s\,\pi}{2\epsilon}\Big)\,
  \dfrac{\bm s}{(\bm s^{\top}\bm s)^{1-p}}, & \bm s^{\top}\bm s\le \epsilon,
\end{cases}
\]
so that if \(\bm s^{\top}\bm s\le\epsilon\) then \(\bm s^{\top}\Theta(\bm s)\ge (\bm s^{\top}\bm s)^{p}-\epsilon^{p}\),
and if \(\bm s^{\top}\bm s>\epsilon\) we recover \emph{Case 1}. Hence
\begin{align}
  \label{V-dot}
\dot V
\leq& - \sum_{i=1}^{2}\! \left(\! \|\beta_{i}\|_2\! \left(\bm y_i^\top \bm y_i \right)^2 \!\!+\! \|\gamma_{i}\|_2\! \left( \bm y_i^{\top} \bm y_i \right)^p \!\!+\! \|c_i\|_2 \!\left(\bm\iota_i^{\top}\bm\iota_i\right)^{2} \!\! \right. \nonumber\\
&\left. + k_i\!\left(\bm\iota_i^{\top}\bm\iota_i\right)^{p}\right) \!\!+\!\! \left(\!\frac{\|\beta_2\|_2}{2}\!+\!\bar\sigma_{1}\!\right)\!\|\bm y_2 \|^4 \!\!+\!\! \left( \!\frac{\|\beta_1\|_2}{4}\! + \!\bar\sigma_{2}\!\right)\! \|\bm y_1\|^4  \nonumber\\
& +\left(\! \frac{\|c_1\|_2}{4}\!+\!\bar\sigma_{3}\!\right)\|\bm \iota_1\|^4 \!+\! \frac{3 \|c_2\|_2}{4}\|\bm \iota_2 \|^4 \!+\! \frac{\bar{\bm\xi}^\top \bar{\bm\xi}}{4 \|h\|_2} \!+\! c_f \nonumber\\
& + \frac{\bar c_{y2}^2}{\beta_2}+ \frac{\bar \sigma_{5}^2}{\beta_1}+ \frac{\bar \sigma_{6}^2}{c_1} + \frac{\bar \sigma_{7}^2}{c_2}\!+\! \bar c_0 + \frac{a+1}{2}k_1 \epsilon^p \!+\! \frac{b+1}{2}\gamma_1 \epsilon^p \nonumber \\
\le & - \sum_{i=1}^{2} \left(4\|\beta_{i}\|_2 \!\left(\!\frac{1}{2}\bm y_i^\top \bm y_i \!\right)^2 \!+\! 2^p \|\gamma_{i}\|_2 \!\left(\!\frac{1}{2} \bm y_i^{\top} \bm y_i \!\right)^p  \right. \nonumber\\
& \quad \left. \!+ 4 \|c_i\|_2\!\left(\!\frac{1}{2}\bm\iota_i^{\top}\bm\iota_i\!\right)^{2} \!+\! 2^p \|k_i\|_2 \!\!\left(\frac{1}{2} \bm\iota_i^{\top}\bm\iota_i\right)^{p} \right)\! \nonumber\\
&\quad +\! \left(\!\frac{\|\beta_2\|_2}{2}\!+\!\bar\sigma_{1}\!\right)\!\! \| \bm y_2\! \|^4 \!+\! \left( \!\frac{\|\beta_1\|_2}{4}\! + \!\bar\sigma_{2}\!\right)\!\! \| \bm y_1 \!\|^4  \nonumber \\
& \quad +\left(\! \frac{\|c_1\|_2}{4}\!+\!\bar\sigma_{3}\!\right)\|\bm \iota_1\|^4 +\! \frac{3\|c_2\|_2}{4}\|\bm \iota_2 \|^4 \!+\! \frac{\bar{\bm\xi}^\top \bar{\bm\xi}}{4 \|h\|_2} \!+\! c_f \nonumber \\
&\quad + \frac{\bar c_{y2}^2}{\|\beta_2\|_2} + \frac{\bar \sigma_{5}^2}{\|\beta_1\|_2}+ \frac{\bar \sigma_{6}^2}{\|c_1\|_2} + \frac{\bar \sigma_{7}^2}{\|c_2\|_2}\!+\! \bar c_0 + \| k_1 \|_2 \epsilon^p \nonumber \\
&\quad + \|\gamma_1\|_2 \epsilon^p.
\end{align}

Above all, it can be concluded that \eqref{V-dot} always hold. From Lemma 3, one obtain that
\begin{equation}\label{eq:Vdot_ft}
\dot V \le -l V^{p} - m V^{2} + n ,
\end{equation}
with
\begin{align}
l& \coloneqq
\min\left\{2^{p} \|k_1\|_2, 2^{p}\|\gamma_{1}\|_2, 2^{p}\|k_2\|_2, 2^{p}\|\gamma_{2}\|_2\right\}, \nonumber \\
m &\coloneqq
\min\left\{\frac{3 \|c_1\|_2}{4}\! - \!\bar\sigma_{3}, \frac{3 \|\beta_1\|_2}{4}\! - \!\bar\sigma_{2}, \frac{\|c_2\|_2}{4}, \frac{\|\beta_{2}\|_2}{2}-\bar\sigma_1 \right\}, \nonumber \\
n& \coloneqq
\frac{\bar{\bm\xi}^\top \bar{\bm\xi}}{4 \|h \|_2} \!+\! c_f \!+\! \frac{\bar c_{y2}^2}{\|\beta_2\|_2} \!+\! \frac{\sigma_{s2}^2}{\|c_2\|_2} \!+\! \bar c_0 \!+\! \|k_1\|_2 \epsilon^p \!+\! \|\gamma_1\|_2 \epsilon^p. 
\end{align}
By \cite{proof}, one has
\begin{equation}\label{eq:V_bound}
V(\varsigma)\;\le\;
\min\!\left\{
\left(\frac{n}{(1-\omega)l}\right)^{\!1/p},
\;
\left(\frac{n}{(1-\omega)m}\right)^{\!1/2}
\right\},
\end{equation}
with $0<\omega<1$, and the settling time $T$ meets 
\begin{equation}\label{eq:T_bound}
T \;\le\; \frac{1}{l}\,\frac{1}{\omega(1-p)} \;+\; \frac{1}{\omega m}.
\end{equation}
\end{proof}

\section{Experimental Results}\label{sec:exp}
\begin{figure}[htbp]
  \centering
  \includegraphics[width=3.4in]{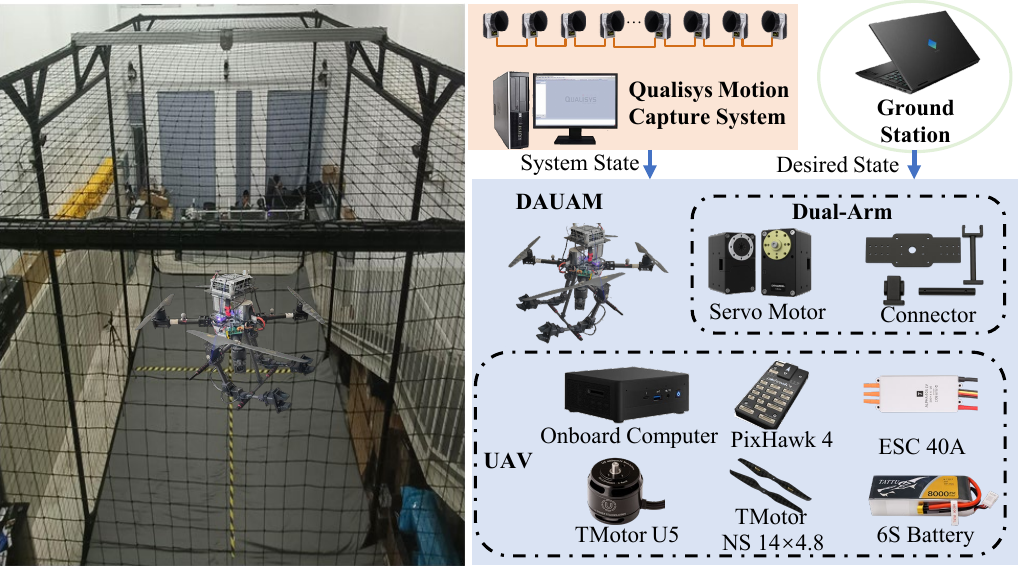}
  \caption{Hardware experimental platform.}
  \label{fig:tb}
  \vspace{-0.2cm}
\end{figure}

\subsection{Experimental platform}
The hardware testbed of the DAUAM system is illustrated in Fig. \ref{fig:tb}. It consists of a multirotor platform carrying an onboard computer and a lightweight dual-arm robotic manipulator. The motion capture system is composed of 14 Qualisys cameras, which provide the multirotor's position and translational velocity to the ground station via a LAN (Local Area Network) using a ROS (Robot Operating System)-based communication interface. A PixHawk flight controller is connected to the onboard computer through the MAVROS middleware, enabling bidirectional exchange of state and command data. The onboard NUC computer and the ground workstation operate on 64-bit Ubuntu 20.04 and 64-bit Ubuntu 18.04, respectively. The control algorithm is executed on the ground workstation with an overall update rate of 100~Hz, and the resulting control commands are sent to the onboard computer over a 5G-band WIFI link. For the dual-arm manipulator, Dynamixel AX- and XM-series actuators are employed.
The dual-arm joints are directly driven by the onboard computer. The main physical parameters of the overall experimental platform are summarized as follows:
\begin{align}
m_t&=4.85~\mathrm{kg}, g=9.8~\mathrm{m/s^2},  \nonumber \\
L_1 &= 0.15~\mathrm{m}, L_2 = 0.05~\mathrm{m}, L_3 = 0.17~\mathrm{m}. \nonumber 
\end{align}
In practical test, the control gains are set as follows:
\begin{align}
&c_1=\mathrm{diag} \left( \left[1.0, 1.0, 1.2 \right] \right), k_1=\mathrm{diag} \left( \left[1.0, 1.0, 1.2 \right] \right), \nonumber\\
&\beta_1=\mathrm{diag} \left( \left[1.0, 1.0, 1.2 \right] \right), \gamma_1=\mathrm{diag} \left( \left[1.0, 1.0, 1.2 \right] \right), \nonumber \\
&c_2=\mathrm{diag} \left( \left[1.0, 1.0, 1.2 \right] \right), k_2=\mathrm{diag} \left( \left[0.8, 0.8, 1.2 \right] \right), \nonumber\\
&\beta_2=\mathrm{diag} \left( \left[0.6, 0.6, 0.7 \right] \right), \gamma_2=\mathrm{diag} \left( \left[1.0, 1.0, 1.2 \right] \right),  \nonumber \\
&k=\mathrm{diag} \left( \left[3.0, 3.0, 4.0 \right] \right), p =0.75, \nonumber \\
&\epsilon = 1.0 \times 10^{-4}, \varepsilon=0.3, \zeta=0.5, \rho=0.1, \nonumber \\
&\sigma=0.05, \delta=0.1, N_0=3, N_1=4, N_3=3, \nonumber \\
&\Gamma_0 = \mathrm{diag} \left( \left[100, 100, 100, 100 \right] \right), \nonumber \\
&\Gamma_1 = \mathrm{diag} \left( \left[100, 100, 100, 100, 100 \right] \right). \nonumber
\end{align}
The control gains of the baseline controller $\bm U_c = - K_p \bm e_1 - K_i \int_0^t \bm e_1 - K_d \dot{\bm e}_1 + m_t \ddot{\bm p}_d$ are set as follows:
\begin{align}
K_p &= \mathrm{diag} \left( \left[8.0, 8.0, 10.0 \right] \right), K_i = \mathrm{diag} \left( \left[1.5, 1.5, 8.0 \right] \right), \nonumber \\
K_d &= \mathrm{diag} \left( \left[10.0, 10.0, 13.0 \right] \right). \nonumber 
\end{align}


\subsection{Experimental Results}
\vspace{-0.2cm}
\subsubsection{Experiment 1 -- Elliptic Trajectory Tracking}
This experiment evaluates the performance of the proposed controller in tracking an elliptic reference trajectory while the dual-arm is moving. The corresponding results are summarized in Fig.~\ref{fig:exp1} and Table~\ref{tab:error}. Specifically, Fig.~\ref{fig:exp1}a depicts the multirotor position, while Fig.~\ref{fig:exp1}b shows the control inputs in three directions, and the NN-based approximation of the friction term is provided in Fig.~\ref{fig:exp1}c. The maximum and mean position errors of the multirotor, as well as the percentages of error reduction of the proposed ET-NN method compared with the baseline method in the three directions, are summarized in Table~\ref{tab:error}. It is evident that both control methods can stabilize the system near the desired trajectory, with the proposed method exhibiting noticeably better trajectory tracking accuracy.

\begin{figure*}[htbp]
  \centering
  \clearcaptionsetup{figure}
  \clearcaptionsetup{subfloat}
  \captionsetup[subfloat]{labelsep=period}
  \subfloat[\small Multirotor position.]{\includegraphics[width=0.30\textwidth]{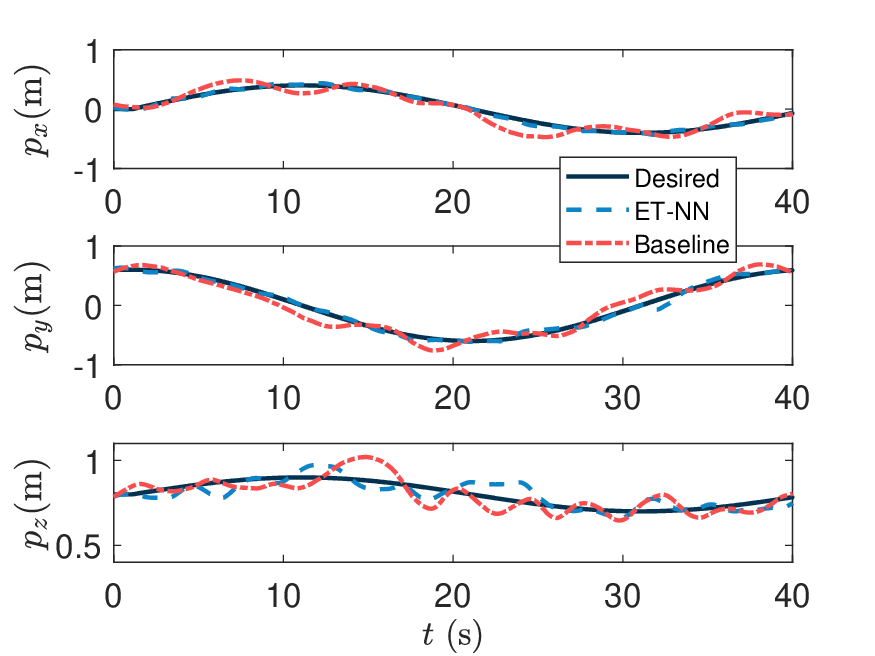}}\hspace{0.0mm} \vspace{-0.2cm}
  \captionsetup[subfloat]{labelsep=period}
  \subfloat[\small Force inputs.]{\includegraphics[width=0.30\textwidth]{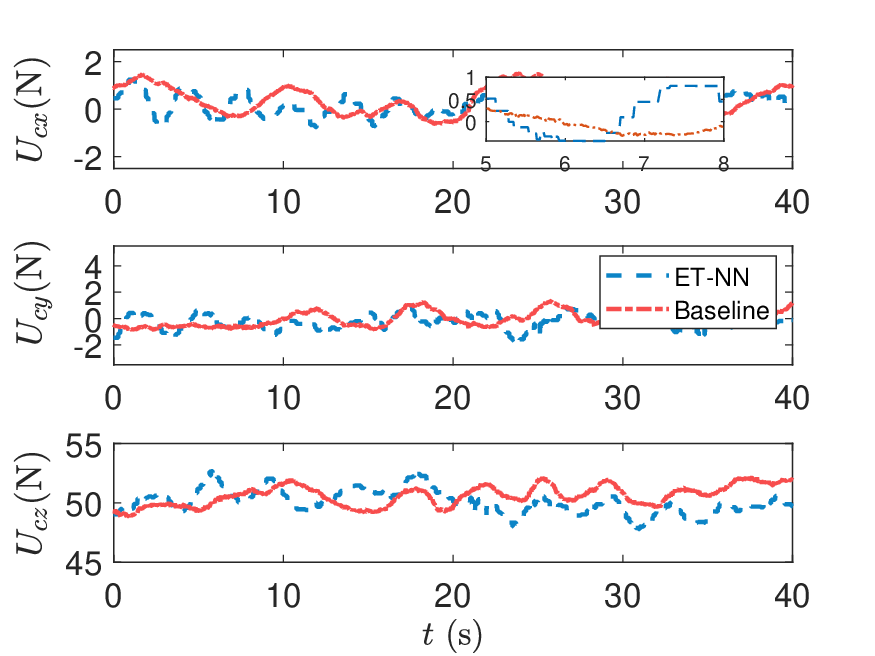}}\hspace{0.0mm}
  \captionsetup[subfloat]{labelsep=period}
  \subfloat[\small NN outputs.]{\includegraphics[width=0.30\textwidth]{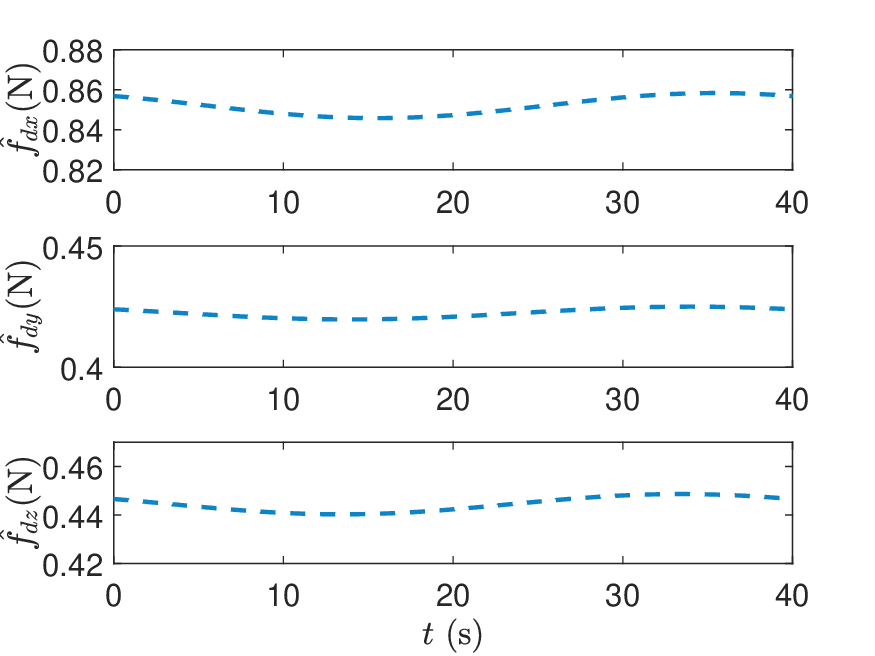}}\hspace{0.0mm}
  \caption{Results of \emph{Experiment 1}.}
  \vspace{-0.2cm}
  \label{fig:exp1}
\end{figure*}

\begin{table}[htbp]\scriptsize
  \centering
  \captionsetup{font={small}}
  \caption{Position Errors of \emph{Experiment 1/2}}
  \renewcommand\arraystretch{1.45}
  \resizebox{\linewidth}{!}{%
  \begin{tabular}{c|c|ccc|ccc}
    \hline\hline
    \multirow{2}{*}{Error} & \multirow{2}{*}{Method} &
    \multicolumn{3}{c|}{\emph{Experiment 1}} &
    \multicolumn{3}{c}{\emph{Experiment 2}} \\
    & & $x~(\mathrm{m})$ & $y~(\mathrm{m})$ & $z~(\mathrm{m})$
      & $x~(\mathrm{m})$ & $y~(\mathrm{m})$ & $z~(\mathrm{m})$ \\ \hline
    \multirow{3}{*}{Max} & ET-NN &
      $\mathbf{0.0663}$ & $\mathbf{0.1672}$ & $\mathbf{0.1151}$ &
      $\mathbf{0.0600}$ & $\mathbf{0.0864}$ & $\mathbf{0.1437}$ \\ \cline{2-8}
    & Baseline &
      $0.2383$ & $0.2510$ & $0.1464$ &
      $0.1876$ & $0.2690$ & $0.2035$ \\ \cline{2-8}
    & Reduced &
      $\mathbf{72.18\%}$ & $\mathbf{33.52\%}$ & $\mathbf{21.39\%}$ &
      $\mathbf{68.20\%}$ & $\mathbf{67.88\%}$ & $\mathbf{29.39\%}$ \\ \hline
    \multirow{3}{*}{Mean} & ET-NN &
      $\mathbf{0.0172}$ & $\mathbf{0.0316}$ & $\mathbf{0.0378}$ &
      $\mathbf{0.0190}$ & $\mathbf{0.0295}$ & $\mathbf{0.0400}$ \\ \cline{2-8}
    & Baseline &
      $0.0841$ & $0.0773$ & $0.0430$ &
      $0.0593$ & $0.0818$ & $0.0673$ \\ \cline{2-8}
    & Reduced &
      $\mathbf{79.57\%}$ & $\mathbf{59.08\%}$ & $\mathbf{12.09\%}$ &
      $\mathbf{68.00\%}$ & $\mathbf{63.96\%}$ & $\mathbf{40.52\%}$ \\ \hline
  \end{tabular}%
  }
  \label{tab:error}
\end{table}
\vspace{-0.2cm}

\subsubsection{Experiment 2 -- Figure-Eight Trajectory Tracking}
To further investigate the robustness of the proposed controller under different reference motions, a second experiment is conducted in which the multirotor is required to track a figure-eight trajectory. The experimental results are provided in Fig.~\ref{fig:exp2} and Table~\ref{tab:error}. As shown in Fig.~\ref{fig:exp2}a, the proposed method maintains accurate tracking of the figure-eight curve, whereas the baseline controller suffers from larger position deviations. The corresponding control inputs and NN outputs are presented in Fig.~\ref{fig:exp2}b and Fig.~\ref{fig:exp2}c, respectively. The illustrative video of experimental records is available at https://youtu.be/p55-FsldKbo.

\begin{figure*}[htbp]
  \centering
  \clearcaptionsetup{figure}
  \clearcaptionsetup{subfloat}
  \captionsetup[subfloat]{labelsep=period}
  \subfloat[\small Multirotor position.]{\includegraphics[width=0.30\textwidth]{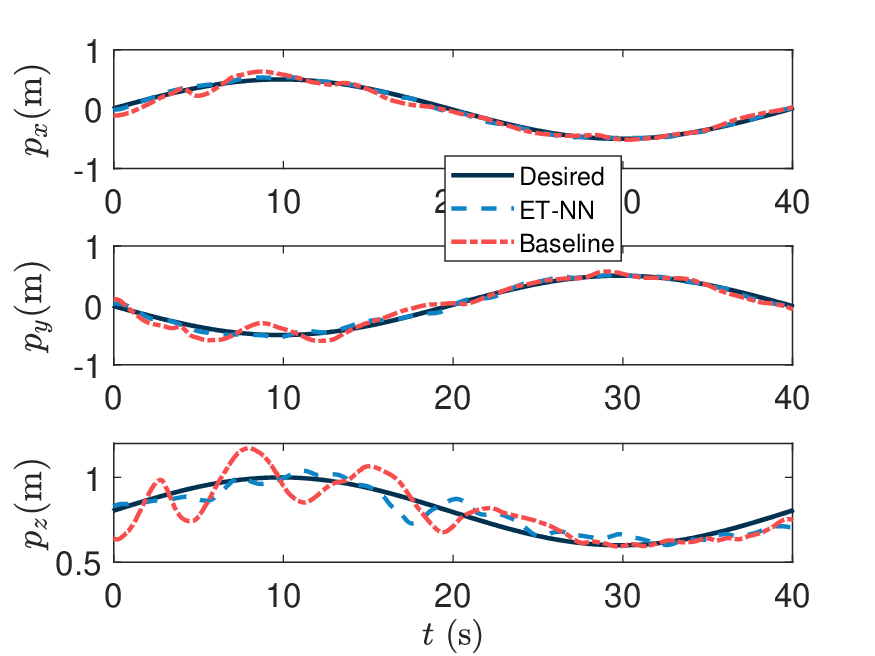}}\hspace{0.0mm}\vspace{-0.2cm}
  \captionsetup[subfloat]{labelsep=period}
  \subfloat[\small Force inputs.]{\includegraphics[width=0.30\textwidth]{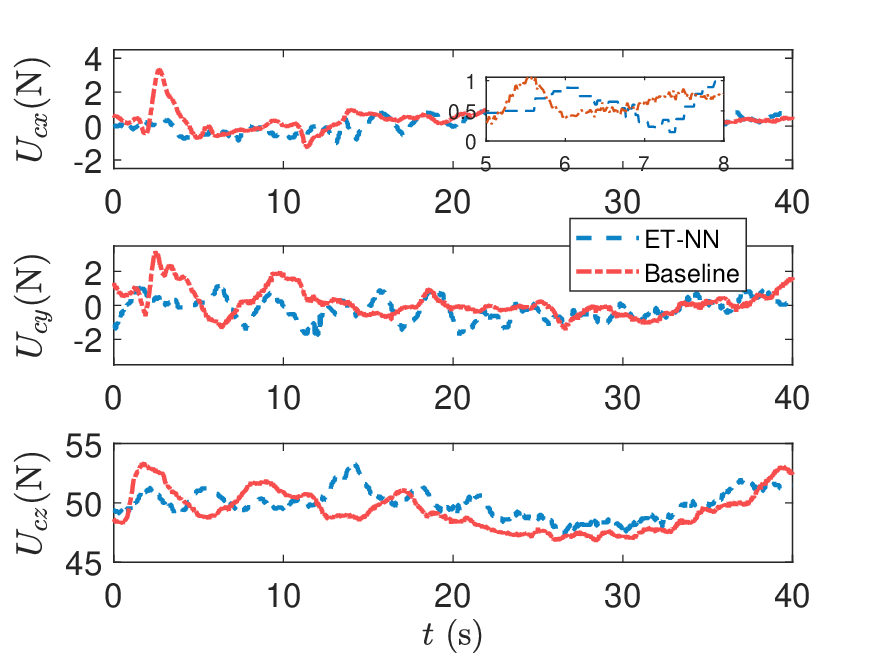}}\hspace{0.0mm}
  \captionsetup[subfloat]{labelsep=period}
  \subfloat[\small NN outputs.]{\includegraphics[width=0.30\textwidth]{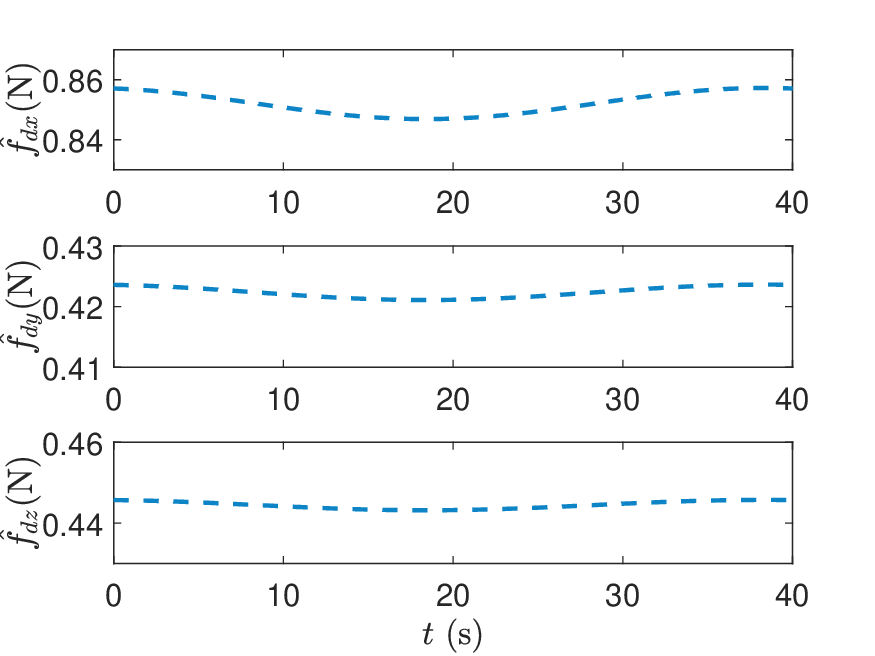}}\hspace{0.0mm}
  \caption{Results of \emph{Experiment 2}.}
  \vspace{-0.2cm}
  \label{fig:exp2}
\end{figure*}

\section{Conclusion}\label{sec:con}
\vspace{-0.2cm}
The paper address the trajectory tracking control of a DAUAM under strong arm-platform coupling, lumped uncertainties, and limited communication resources. An adaptive event-triggered controller based on command-filtered backstepping and neural network approximation is developed to avoid the explosion of complexity, compensate unknown dynamics, and reduce control-update frequency. Lyapunov-based analysis established boundedness of all closed-loop signals and fixed-time convergence of the tracking error to a small neighborhood of the desired trajectory, which is further corroborated by experiments on a self-built DAUAM platform. Future research directions include extending the framework to contact-rich manipulation tasks and cooperative multi-UAV manipulation scenarios.

\bibliography{ifacconf}             
                                                   







\end{document}